\documentclass[sn-mathphys,Numbered]{sn-jnl}

\usepackage{graphicx}%
\usepackage{multirow}%
\usepackage{amsmath,amssymb,amsfonts}%
\usepackage{amsthm}%
\usepackage{mathrsfs}%
\usepackage[title]{appendix}%
\usepackage{xcolor}%
\usepackage{textcomp}%
\usepackage{manyfoot}%
\usepackage{booktabs}%
\usepackage{algorithm}
\usepackage{algorithmicx}%
\usepackage{algpseudocode}%
\usepackage{listings}%

\newtheorem{lemma}{Lemma}
\usepackage{amsmath}

\begin{document}

\title[Article Title]{An Incentive Mechanism for Federated Learning Based on Multiple Resource Exchange}


\author[1,2]{\fnm{Ruonan} \sur{Dong}}\email{dongrn@cumt.edu.cn}\equalcont{These authors contributed equally to this work.}

\author[1,2]{\fnm{Hui} \sur{Xu}}\email{xuhui@cumt.edu.cn}
\equalcont{These authors contributed equally to this work.}

\author[1,2]{\fnm{Han} \sur{Zhang}}\email{hanzhangl@cumt.edu.cn}

\author*[1,2]{\fnm{GuoPeng} \sur{Zhang}}\email{gpzhang@cumt.edu.cn}

\affil[1]{\orgdiv{Engineering Research Center of Mine Digitalization, Ministry of Education}, \orgname{China University of Mining and Technology}, \orgaddress{ \city{Xu Zhou}, \postcode{221116},  \country{China}}}

\affil[2]{\orgdiv{School of Computer Science and Technology}, \orgname{China University of Mining and Technology}, \orgaddress{ \city{Xu Zhou}, \postcode{221116}, \country{China}}}



\abstract{Federated Learning (FL) is a distributed machine learning paradigm that addresses privacy concerns in machine learning and still guarantees high test accuracy. However, achieving the necessary accuracy by having all clients participate in FL is impractical, given the constraints of client local computing resource. In this paper, we introduce a multi-user collaborative computing framework, categorizing users into two roles: model owners (MOs) and data owner (DOs). Without resorting to monetary incentives, an MO can encourage more DOs to join in FL by allowing the DOs to offload extra local computing tasks to the MO for execution. This exchange of "data" for "computing resources" streamlines the incentives for clients to engage more effectively in FL. We formulate the interaction between MO and DOs as an optimization problem, and the objective is to effectively utilize the communication and computing resource of the MO and DOs to minimize the time to complete an FL task. The proposed problem is a mixed integer nonlinear programming (MINLP) with high computational complexity. We first decompose it into two distinct subproblems, namely the client selection problem and the resource allocation problem to segregate the integer variables from the continuous variables. Then, an effective iterative algorithm is proposed to solve problem. Simulation results demonstrate that the proposed collaborative computing framework can achieve an accuracy of more than 95\% while minimizing the overall time to complete an FL task.}




\maketitle

\section{Introduction}\label{sec1}

In recent years, artificial intelligence (AI) and machine learning (ML) have witnessed tremendous success across various applications, including computer vision, recommender systems, and natural language processing. Typically, training AI models demands significant computing power and large amount of data samples, and is often carried out in centralized clouds or data centers equipped with robust computing and storage capabilities \cite{8847416}. Advancements in the Internet of Things (IoT)\cite{2020Deep} and 5G network technologies have bestowed user devices with sophisticated sensing, communication, and computing capabilities. User devices can execute more intricate tasks and collaboratively sense information in diverse contexts. Such tasks hinge on extensive data analysis and model training. However, traditional centralized learning proves ineffective for training AI models due to the high cost associated with collecting large amounts of user data, leading to extremely heavy traffic loads in communication networks. Moreover, the concerns of individual privacy and information security add to the complexity of sharing user data.

In response to these challenges, Google Inc. \cite{2016Communication} introduced the Federated Learning (FL) in 2016, presenting a distributed learning paradigm. In the framework of FL  \cite{9084352}, user devices can maintain their private data and update model parameters through local training, thus enabling data isolation and eliminating the need to transmit raw data to a central server. In detail, the user devices participating FL are required to train an AI model locally and update the model parameters. The local updates are accomplished after multiple iterations to hasten convergence. Following the completion of local training \cite{9060868}, users upload the model parameters to a model server. Then, the model server, in turn, aggregates the local model parameters from all users to derive the updated global model parameters.

Although FL may result in some performance degradation, it is typically deemed acceptable as data privacy \cite{ 2019Federated} is assured. However, the resource-intensive nature of model training and parameter uploading poses challenges in terms of computing and communication resources for user devices. In addition, the diverse computational capabilities of different user devices may pose implementation challenges in resource allocation problems \cite{9509294}.
Therefore, the application scenario of FL aligns closely with resource-constrained network edges. 

In order to alleviate the scarcity of computing resources for user devices and data scarcity for model servers, we propose a resource sharing method based on the resource exchange between a model owner (MO) and multiple data owners (DOs). The MO and DOs just correspond to an edge computing server and multiple user devices in edge environment. To incentivize the active participation of DOs, the MO should address the challenges of data and computing resource sharing by undertaking a portion of the computing loads of DOs as the compensation for the DOs to train local model. The contribution of this paper is as follows:
\begin{itemize}
    \item[1)] We formulate the interaction between MO and DOs as an optimization problem, and the objective is to effectively utilize the communication and computing resource of the MO and DOs to minimize the time to complete an FL task.
    \item[2)] The proposed problem is a mixed integer nonlinear programming (MINLP) with high computational complexity. We first decompose it into two distinct subproblems, namely the client selection problem and the resource allocation problem to segregate the integer variables from the continuous variables. Then, an effective iterative algorithm is proposed to solve problem.
    \item[3)] Simulation results demonstrate that the proposed collaborative computing framework can achieve an accuracy of more than 95\% while minimizing the overall time to complete an FL task.
\end{itemize}

The rest of the paper is organized as follows: 
Section II reviews the relevant literature for this paper, providing a comprehensive overview of existing studies. In Section III, the system model is presented, followed by the formulation of the time optimization problem in Section IV. Section V delves into the analysis of the problem and presents the optimal solution. Section VI assesses the system's performance through simulation, while the concluding insights are encapsulated in Section VII.

\section{Related works}\label{sec2}
In order to solve the resource allocation problem in federated learning, many scholars have proposed game-theoretic \cite{2004An} approaches and introduced incentives \cite{9247530} between mobile edge servers and users to reach a Nash equilibrium for a relatively desirable outcome. 
Authors such as G. Xiao \cite{9359188} and L. Dong \cite{9299689} have employed the Stackelberg game model to collectively investigate utility maximization. In the former's work, a federated learning framework based on crowdsourcing was designed, establishing an incentive mechanism for multiple devices to engage in federated learning. Clients participating in the process and mobile edge servers interact through an application platform to construct high-quality learning models. In the latter study, a market-oriented approach was considered to motivate users to participate in federated learning, integrating their computational energy consumption into the model.
Y. Jiao \textit{et al}. authors \cite{9094030} introduced an auction-based market model to incentivize data owners' participation in federated learning. They devised two auction mechanisms: a reverse multidimensional auction mechanism and a deep reinforcement learning-based auction mechanism, aiming to maximize the social welfare of the federated learning service market.

Some scholars have proposed offloading local tasks to the edge so that clients can complete federated learning tasks more efficiently. 
S. Zhou \textit{et al}. \cite{2021Machine} conducted an in-depth study and analysis of offloading strategies for lightweight user-mobile edge computing tasks using a machine learning approach. 
J. Ren \textit{et al}. authors \cite{8728285} introduced a computation offloading scheme based on federated learning for edge computing-enabled IoT platforms. This scheme aims to augment the constrained capacity of IoT devices by offloading computationally intensive tasks to edge nodes. The real-time determination of computation offloading decisions involves addressing complex resource management challenges. Additionally, to minimize the transmission cost between IoT devices and edge nodes, agents undergo distributed training using a federated learning mechanism.
S. Kim \textit{et al}. \cite{9590495} explored a federated learning-based tactical edge network platform. In this setup, each individual device autonomously decides on offloading strategies for its tactical tasks within a resource-constrained network environment. The motivation behind these decisions is driven by model training, with the overall objective of efficient collaboration in the execution of a multitude of computational tasks.

\section{System Model}

\vspace{-4 mm}
\begin{figure}[H]
  \centering
  \includegraphics[scale=.3]{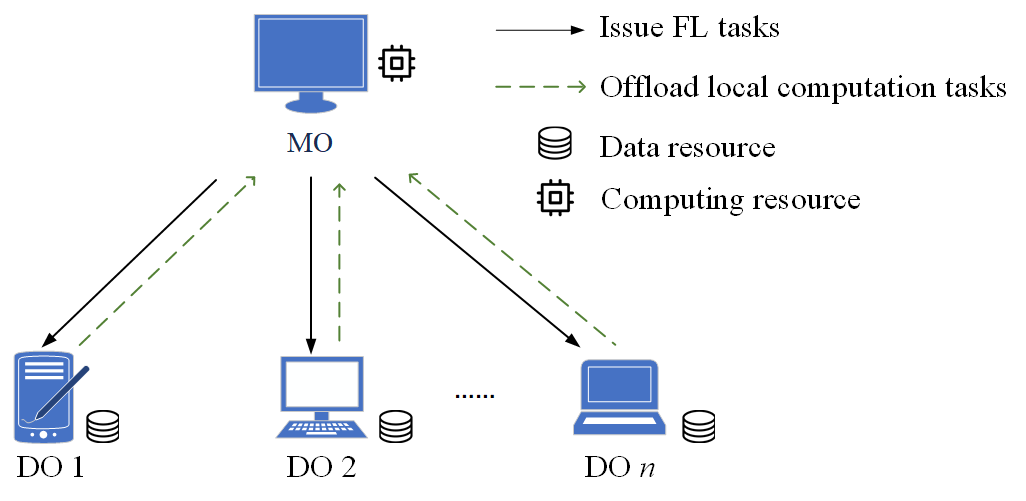} 
  \vspace{-2 mm}
  \caption{The system model of federated learning.} \label{Multi-User Cooperative}
\end{figure}
\vspace{-3 mm}

The considered FL system is illustrated in Fig.\ref{Multi-User Cooperative}. There exists an MO and a set $\mathcal{N}=\left \{1,...,N  \right \} $ of $N$ DOs, and the MO can communicate with each of the DOs via direct wireless links. The MO has an AI model to train and holds a certain amount of computing power, while each DO has a certain amount of data samples. The MO lacks data for training the AI model but has a large amount of idle computing resources, while each DO holds a large amount of data sample but lack local computing resources.
To protect data privacy, the MO can employ FL to utilize the data of the DOs and undertake a portion of computing tasks of the DOs to reduce their computing loads.

\subsection{FL Model}\label{subsec2}
Each DO $n$, $\forall n\in \mathcal{N}$, stores a local dataset $\mathcal{D}_n$ of size $|\mathcal{D}_n|$. The data size of the $N$ DOs is give by $\mathcal{D}= { {\textstyle\sum_{n=1}^{N}}}\mathcal{D}_n$.
The loss function of any DO $n$ is $f_i(\mathbf{w}^t)=\ell (\mathbf{x}_i^n,y_i^n;\mathbf{w}^t)$, where $\left \{\mathbf{x}_i^n,y_i^n\right \}$ represents the $i^{th}$ data sample of DO $n$. The loss function of DO $n$ with the model parameter $\mathbf{w}$ is given by \cite{9233403}

\begin{equation}
    F_n(\mathbf{w}^t ):=\frac{1}{D_n}\sum_{i\in\mathcal{D}_n }^{} f_i (\mathbf{w}^t).
\end{equation}
Then, we can solve the optimal model by minimize the following global loss function 

 \begin{equation}\label{loss1}
\min_{\mathbf{w}^t} F(\mathbf{w}^t ):=\sum_{n=1}^{N} \frac{D_n}{D}  F(\mathbf{w}^t ).
\end{equation}
When the loss function $f_i (\mathbf{w}^t)$ is convex, we can use iterative methods to find the optimal solution. Considering the inherent complexity of machine learning models, gradient descent techniques are usually used to solve the problem \eqref{loss1}.

\subsection{Time Consumption Model}
Before engaging in an FL task, the DOs may process some local computing tasks, which is represented by $\varphi_n=\left (  h_n,c_n\right )$. $h_n$ denotes the amount of data for the DO local task, and  $c_n$ denotes the number of CPU cycles required for DO $n$ to train per data sample. So the amount of local task is adjust to $C_n^\text{L}=c_n h_n$. Let $f_n^\text{L}$ (cycles/s) denote the frequency of DO $n$ to deal with its local tasks. We use $T_n^\text{Ex}$ to represent the time of the DOs to complete one local training session. Let $f^\text{max}_0$ denote the maximum CPU frequency of a DO. The following constraint should be satisfied 
\begin{equation}
    0 \leq f_n^\text{L}\leq f^\text{max}_0,\ \forall n\in \mathcal{N}. 
\end{equation}

After participating the FL task issued by the MO, each DO $n$ should raise the CPU frequency by $f_n^\text{F}$ cycles/s in order to complete its local tasks and the FL task within $T_n^\text{Ex}$. 
Let $I_n$ denote the required number of local iterations for DO $n$. The CPU cycles required for DO $n$ to complete a session of local training is given by $C_n^\text{F}=c_nI_nD_n$. Therefore, DO $n$ can estimate its local computation time in the $t^{th}$ training session as

\begin{equation}
    T_n^\text{Ex}=\frac{c_nI_nD_n}{f_n^\text{F}},\ \forall n\in \mathcal{N}.
\end{equation}

In order to reduce the overall training time, the MO is willing to devote part of its computing resources and allows the DOs to offload part of their local tasks to the MO for remote execution. Let $C_n^\text{off}$ denote the computing load of the tasks offloaded by DO $n$. Let $C^\text{max}$ denote the maximum computing power DO $n$.
Each DO $n$ can offload its local tasks to the MO only when $f_n^\text{L}+f_n^\text{F}>f_0^\text{max}$ and $f_n^\text{F}<f_0^\text{max}$. Then, the amount of computing load offloaded by DO $n$ to the MO is given by

\begin{equation}
    C_n^\text{off} = C_n^\text{L}+C_n^\text{F}-C^\text{max}. \label{coff}
\end{equation}

After each round of local training session, all the DOs should upload the trained local models to the MO. Let $g_n$ denote the channel gain from DO $n$ to the MO. The transmission rate from DO $n$ to the MO is given by \cite{9891799}

 \begin{equation}
    r_n = b_n\log_2(1+\frac{p_n g_n}{N_0b_n} ),\ \forall n\in \mathcal{N}, 
\end{equation}
where $b_n$ is the bandwidth allocated to DO $n$, $p_n$ is the transmission power of DO $n$, and $N_0$ is the spectral density of the Gaussian noise.

It is noted that the data to be transmitted by Do $n$ to the MO is divided into two parts: the parameters of trained local model $\textbf{w}_n$ and the input data of the offloaded tasks. Let $d_n$ denote the size of the input data of the tasks offloaded by Do $n$. Then we have 
 \begin{equation}
   d_n=\frac{C_n^\text{off}}{c_n},\ \forall n\in \mathcal{N}. \label{dn}
\end{equation}

Let $H_n=d_n + |\textbf{w}_n|$ denote the size of the data to be transmitted by Do $n$ to the MO. To ensure correct reception, the condition $r_n T_n^\text{Tx} \ge H_n$ should be satisfied. The lower bound of $T_n^\text{Tx}$ is obtained as

\begin{equation}
    T_n^\text{Tx}\ge\frac{H_n}{r_n} =\frac{d_n+|\textbf{w}_n|} {b_n\log_2(1+\frac{p_n g_n}{N_0b_n} )},\ \forall n\in \mathcal{N}. 
\end{equation}
Let $T_n$ denote the overall time for DO $n$ to complete one local training session. We can express $T_n$ as

\begin{equation}
    T_n= T_n^\text{Ex}+T_n^\text{Tx} = \frac{c_nD_nI_n }{f_n^F} +\frac{H_n}{b_n\log_2(1+\frac{p_n g_n}{N_0b_n} )},\ \forall n\in \mathcal{N}.
\end{equation}

\section{Problem Formulation}\label{sec3}
\begin{figure}[H]
  \centering
  \includegraphics[scale=.37]{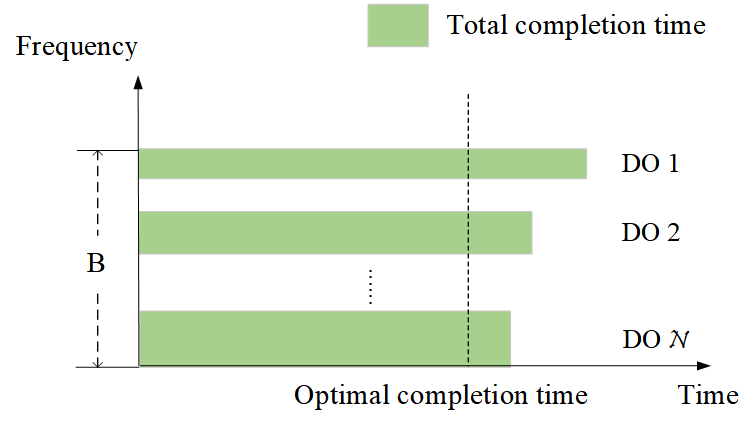} 
  \vspace{-2 mm}
  \caption{Communication and computation resource allocation for the DOs.} \label{completion time}
\end{figure}

As illustrated in Fig. \ref{completion time}, in order to minimize the overall training time, we must minimize the time $T_n$ for completing each round of global training. Let $a_n\in \left \{ 1,0 \right \} $ denote whether or not DO $n$ is selected, $p_n^\text{max}$  the maximum transmit power of DO $n$, and $T^{max}$ the time threshold for completing one round of training session. The time minimization problem can be mathematically expressed as
\begin{align}
    \min _{\mathbf{F,B,P,A} }\max_{} \; & \; T_n, \label{min} \\ 
    \mbox{s.t.}\;\;\;\;\; &  \sum_{n=1}^{N}a_n(f_n^\text{L}+f_n^\text{F}-f_0^\text{max})\le F\tag{\ref{min}.1},\ \forall n\in \mathcal{N}, \label{fn}\\
    & \sum_{n=1}^{N}a_nb_n\le B,\ \forall n\in \mathcal{N},\tag{\ref{min}.2}  \label{bn}   \\
    & 0 < T_n \le T^\text{max},\ \forall n\in \mathcal{N}, \tag{\ref{min}.3}  \label{T} \\
    &  0<f_n^\text{F}<f_0^\text{max},\ \forall n\in \mathcal{N},\tag{\ref{min}.4}\label{fmax}\\
    &0<p_n \le p_n^\text{max},\ \forall n\in \mathcal{N},\tag{\ref{min}.5}\label{pmax}\\
    &a_n\in\left \{ 0,1 \right \},\ \forall n\in \mathcal{N}, \tag{\ref{min}.6}  \label{an}
\end{align}
where $\mathbf{F} =\begin{bmatrix} f_1,...,f_n \end{bmatrix}^T$, $\mathbf{B} =\begin{bmatrix}
  b_1,...,b_n \end{bmatrix}^T$, $\mathbf{P} =\begin{bmatrix} p_1,...,p_n \end{bmatrix}^T$,
$\mathbf{A}=\begin{bmatrix} a_1,...,a_n \end{bmatrix}^T$.
Constraint \eqref{fn} states that the amount of computing tasks offloaded by all DOs cannot exceed the maximum computing capacity that the MO can process.
Constraint \eqref{bn} states that the total bandwidth allocated for transmitting data from the DOs to the MO cannot exceed the current available bandwidth $B$ of the system.

\section{Optimal Resource Allocation for Minimizing Training Time}
Problem \eqref{min} falls into the category of mixed integer nonlinear programming (MINLP), making a direct attainment of an optimal solution challenging. Nevertheless, the problem can be decomposed into two distinct subproblems, namely the client selection problem and the resource allocation problem, by segregating the integer variables from the continuous variables.

\subsection{Client Selection Algorithm}
In the resource exchange based FL system, any DO $n$ is willing to participate in the FL task initiated by the MO only when it takes on the new FL task $C_n^\text{F}$ and requires more computing power than it currently has $C^\text{max}$, i.e., $ C_n^\text{off} = C_n^\text{L}+C_n^\text{F}-C^\text{max}>0$. Then we give the client selection algorithm as follows.
\begin{algorithm}
    \renewcommand{\algorithmicrequire}{\textbf{Input:}}
    \renewcommand{\algorithmicensure}{\textbf{Output:}}
    \caption{DO Selection Algorithm}
    \label{DOclient}
    \begin{small}
  \begin{algorithmic}[1]
    
    \Require  $C_n^\text{L}, C_n^\text{F}, C^\text{max}$.
\Ensure Client selection indicator $a_n$, the data size of the offloaded local tasks $d_n$.
		\Repeat \\ 
		\state Obtain $C_n^\text{off}$ according to eq. \eqref{coff}.
		 \If{ $C_n^\text{off}>0$, }
		    \State Set $a_n=1$, and obtain $d_n$ using eq. \eqref{dn}.
		\EndIf
		\Until Each client has been traversed once.
  \end{algorithmic}
  \end{small}
\end{algorithm}

\subsection{Resource Allocation Algorithm}
After performing the client selection algorithm, i.e., \textbf{Algorithm 1}, the objective function of problem \eqref{min} can be transformed into a continuous non-convex problem as follows:
\begin{align}
    \min _{\mathbf{F,B,P} }\max_{} \; & \; T_n, \label{min1} \\ 
    \mbox{s.t.}\;\;\;\;\; &  \sum_{n=1}^{N}(f_n^\text{L}+f_n^\text{F}-f_0^\text{max})\le F\tag{\ref{min1}.1},\ \forall n\in \mathcal{N},  \label{fn1}\\
    & \sum_{n=1}^{N}b_n\le B,\ \forall n\in \mathcal{N}, \tag{\ref{min1}.2}  \label{bn1}   \\
    & 0 < T_n \le T^\text{max},\ \forall n\in \mathcal{N}, \tag{\ref{min1}.3}  \label{T1} \\
    &  0<f_n^\text{F}<f_0^\text{max},\ \forall n\in \mathcal{N},\tag{\ref{min1}.4}\label{fmax1}\\
    &0<p_n \le p_n^\text{max},\ \forall n\in \mathcal{N}.\tag{\ref{min1}.5}\label{pmax1}
\end{align}

\begin{lemma}
    When $T_1 = ... = T_n=\bar{T}$, problem \eqref{min1} reaches its optimal solution. \label{lemma1}
\end{lemma}

\begin{proof}
    See Appendix A.
\end{proof}

According to Lemma \ref{lemma1}, we introduce an auxiliary variable $\bar{T}$ to simplify the objective function of problem \eqref{min1} as in \cite{shen2022joint}. Then problem \eqref{min1} is rewritten as

\begin{align}
    \min _{\mathbf{F,B,P},\bar{T}}\; & \; \bar{T}, \label{min2} \\ 
    \mbox{s.t.}\;\;\;\;\; &\eqref{fn1},\eqref{bn1},\eqref{T1},\eqref{fmax1},\eqref{pmax1}\notag\\
    & \bar{T}\ge \left \{ \frac{c_nD_nI_n }{f_n^F} +\frac{H_n}{b_n\log_2(1+\frac{p_n g_n}{N_0b_n} )} \right \},\ \forall n\in \mathcal{N}. \tag{\ref{min2}.1}\label{cons2}
\end{align}

Due to the presence of non-convex constraints, we propose an iterative algorithm to solve problem \eqref{min2}. The main method is as follows. Firstly, we fix $(\mathbf{B,P})$ and find the optimal solution to ($\mathbf{F}$). Then, we update $(\mathbf{B,P})$ based on the obtained ($\mathbf{F}$) in the previous step.

Given ($\mathbf{B,P}$), problem \eqref{min2} becomes:

\begin{align}
    \min _{\mathbf{F},\bar{T} }\; & \; \bar{T}, \label{Tf} \\ 
  \mbox{s.t.}\;\;\;\;\; &  \sum_{n=1}^{N}(f_n^\text{L}+f_n^\text{F}-f_0^\text{max})\le F,\forall n\in \mathcal{N},\tag{\ref{Tf}.1}  \label{Tfn}\\
    &  0<f_n^\text{F}<f_0^\text{max},\ \forall n\in \mathcal{N}.\tag{\ref{Tf}.2}\label{Tfmax}
\end{align}

It is noted that problem \eqref{Tf} is concave. So we can employ the Lagrange duality method \cite{8644186} to solve it. Let $\lambda \ge 0$ and $\mu \ge 0$ denote the Lagrange multiplier associated with the constraints in \eqref{Tfn} and \eqref{Tfmax}. Then the partial Lagrangian of problem \eqref{Tf} is given by
\begin{align}
     \mathcal{L}(\mathbf{F}, \lambda,\mu )=\frac{c_nD_nI_n }{f_n^F}+\sum_{n=1}^{N}\lambda  (f_n^\text{L}+f_n^\text{F}-f_0^\text{max}-F) +\mu (f_n^F-f_0^\text{max}).
\end{align}
Accordingly, the dual function is given by 
\begin{align}
    \mathcal{G} (\lambda,\mu)= \text{inf} \;\;\mathcal{L}(\mathbf{F, \lambda,\mu}).\label{dule} 
\end{align}
Then, the dual problem is 
\begin{align}
    \max_{\lambda,\mu}\mathcal{G} (\lambda,\mu),\label{dule1} \\ 
    \mbox{s.t.}\;\;\;\;\; &  f_n^\text{F}>0,\forall n\in \mathcal{N}, \tag{\ref{dule1}.1} \label{dule1f}  \\
    & \lambda \ge 0, \mu \ge 0. \tag{\ref{dule1}.2} \label{dule1ml} 
\end{align}

\begin{lemma}
    Problem \eqref{dule1} is convex under given $\lambda, \mu$, and the optimal solution $(f^\text{F})^\ast$ is given by \label{fast}
\end{lemma}

\begin{equation}
    (f_n^\text{F})^\ast=\sqrt{\frac{C_nD_nI_n}{\lambda+\mu}}, \ \forall n\in \mathcal{N}. \label{optimalf}
\end{equation} 

\begin{proof}
    See Appendix B.
\end{proof}

After obtaining the optimal solution of $\mathbf{F}$, problem \eqref{min2} is simplified to
\begin{align}
    \min _{\mathbf{B,P},\bar{T}} \; & \;\bar{T},\label{Tcom} \\ 
    \mbox{s.t.}\;\;\;\;\; &  \sum_{n=1}^{N}b_n\le B,\ \forall n\in \mathcal{N},\tag{\ref{Tcom}.1}  \label{Tcombn}   \\
    &0<p_n \le p_n^\text{max},\ \forall n\in \mathcal{N}.\tag{\ref{Tcom}.2}\label{Tcompmax}
\end{align}

\begin{lemma}
    The objective of problem \eqref{Tcom} monotonically decreases with respect to $p_n$ and $b_n$, and $p_n$ also monotonically decreases with respect to $b_n$. \label{pnbn}
\end{lemma}

\begin{proof}
    See Appendix C.
\end{proof}

According to \textbf{Lemma 3} and \eqref{function of pb}, we can transform problem \eqref{Tcom} into the following form
\begin{align}
    \max _{b_n } \; & \; b_n,\ \forall n\in \mathcal{N}, \label{maxb to minimize t} \\ 
    \mbox{s.t.}\;\;\;\;\; &  \sum_{n=1}^{N}b_n\le B,\forall n\in \mathcal{N},\tag{\ref{maxb to minimize t}.1}  \label{bn4}   \\
    &0<\frac{N_0b_n}{g_n}\left ( 2^\frac{H_n}{T_n^\text{Tx}b_n} -1 \right )\le p_n^\text{max},\forall n\in \mathcal{N},\tag{\ref{maxb to minimize t}.2}\label{pmax4}
\end{align}
where constraint \eqref{pmax4} is specified in \eqref{function of pb}. Considering that constraint \eqref{pmax4} permits the separated optimization of $p_n$ and $b_n$ for each DO, we give the upper bound of $b_n$ as the following form as in \cite{9723627}. 
\begin{equation}
    b_n^\text{max}=\text{arg} \min \left \{ b_n>0 | p_n(b_n)>0 \right \}  , \forall n\in \mathcal{N},\label{argbn}
\end{equation}
where $b_n^\text{max}$ is the maximum of $b_n$ satisfying the constraints \eqref{bn4} and \eqref{pmax4}. The optimization problem \eqref{argbn} is then transformed into
\begin{equation}
    b^{\ast}= \min_{ n\in \mathcal{N}} b_n^\text{max}.\label{transb}
\end{equation}
Problem \eqref{transb} can be solved by the following lemma.

\begin{lemma}
       The optimal solution $b^\ast$ of problem \eqref{transb} is
        \label{bp}
        \begin{align}
            b_n^\ast=
            \begin{cases}
               b_n^\text{min},& p_n=p_n^\text{max},\\
               b_n(\rho ),& 0<p_n<p_n^\text{max}.\label{the optimal of b}
            \end{cases}
        \end{align} 
     where $\rho$ is the Lagrange multiplier, and  $b_n(\rho )$ is the solution to
     \begin{equation}
     \frac{N_0}{g_n} \left ( \textit {e}^{\frac{(\ln_{}{2})H_n}{T_n^\text{Tx}b_n(\rho )}  } -1-\frac{(\ln_{}{2})H_n }{T_n^\text{Tx}b_n(\rho )}\textit {e}^{\frac{(\ln_{}{2})H_n}{T_n^\text{Tx}b_n(\rho )}}   \right )+\rho =0.\label{solution rho} 
   \end{equation}

\end{lemma}
\begin{proof}
See Appendix D.
\end{proof}

Finally, we conclude the solution method to problem \eqref{min2} in the following \textbf{Algorithm 2}.

\begin{algorithm}
\caption{The algorithm to solve problem \eqref{min2}.}
    \label{iterative}
    \begin{small}
  \begin{algorithmic}[1]   
    \State Initialization: Set initial values for $(\textbf{F}^0, \textbf{B}^0,\textbf{P}^0)$. 
    \State The solution is obtained by using eq.\eqref{optimalf} and eq.\eqref{the optimal of b}. 
    \State Set the iteration number $l=0$.
		\Repeat 
		\State With the known $(\textbf{B}^l, \textbf{P}^l)$, obtain the optimal $\textbf{F}^{l+1}$ by \eqref{optimalf}. 
		\State Substitute $(\textbf{F}^{l+1}, \textbf{B}^{l}, \textbf{P}^l)$ into eq. \eqref{min2} and obtain the optimal value for $\bar{T}^l$.
        \State With the known $(\textbf{F}^{l+1})$, obtain $(\textbf{B}^{l+1},\textbf{P}^{l+1})$ by \eqref{the optimal of b}.
        \State Let $l=l+1$.
		\Until constrint \eqref{cons2} converges or the maximum number of iterations $l^\text{max}$ is reached.
  \end{algorithmic}
  \end{small}
\end{algorithm}

\begin{lemma}
       Algorithm 2 converges to a stable solution. \label{converge}
\end{lemma}

\begin{proof}
At the $l$-th iteration, we know that
    \begin{align}
\bar{T}^l=& \bar{T}\;(\mathbf{F}^l,\mathbf{B}^l,\mathbf{P}^l)\\ \notag
\ge& \bar{T}\; (\mathbf{F}^{l+1},\mathbf{B}^l,\mathbf{P}^l)\\ \notag
\ge& \bar{T}\; (\mathbf{F}^{l+1},\mathbf{B}^{l+1},\mathbf{P}^{l})\\ \notag
\ge& \bar{T}\; (\mathbf{F}^{l+1},\mathbf{B}^{l+1},\mathbf{P}^{l+1})\notag.
\end{align}
The first inequality is derived from eq. \eqref{optimalf}. The validity of the second and third inequalities is ensured due to the joint bandwidth and transmit power allocation. Finally, with the optimal value obtained, solving for $\bar{T}$ is performed. As the objective value $\bar{T}$ remains non-negative throughout the optimization, the convergence of the algorithm is ensured.
\end{proof}

The computational complexity of the algorithm analyzed as follows. The complexity of \textbf{Algorithm \ref{DOclient}} is $\mathcal{O}(N)$ since the client selection is linear with respect to the number of DOs $N$. The complexity of to sovle problem \eqref{Tf} is $\mathcal{O}(N_{se})$ according to eq. \eqref{optimalf}, where $N_{se}$ is the number of the selected DOs. The algorithm to solve problem $\eqref{Tcom}$ is with the complexity of $\mathcal{O}(N_{se}\log_{2}({1/\epsilon _1})({1/\epsilon _2}))$, where $\epsilon _1$ and $\epsilon _2$ are respectively the accuracy of solving eq. \eqref{the optimal of b}. As a result, the computational complexity of the overall Algorithm is $\mathcal{O}(N+l^{max}\log_{2}({1/\epsilon _1})({1/\epsilon _2})N_{se})$.

\section{Simulation Results}\label{sec4}
The experiments are carried on the TensorFlow FL framework. The dataset used is the MNIST dataset. Before each global model update, an independent identical distributed (IID) dataset is extracted from the MNIST and allocated to each DO. At the end of each global update, the current dataset of the DO is cleared, and a new one is generated.

We consider a square area of $500 \times 500 m^2$ with $N = 50$ DOs uniformly distributed within it. The path loss is modeled as $-128.1 + 37.6 \log_{}{}  (d)$  (where $d$ is in km). The standard deviation of shadow attenuation is 8 dB, and the noise power spectral density is $N_0$ = -174 dBm/Hz. The total bandwidth of the FL system is $B=2$MHz, and the computing power of the CPU of MO is $F = 2$ GHz. The size of the tranined model is $|\textbf{w}_n|= 24.6$ kb. We set $c_n=20$, and assume that all the DOs transmit the data with the maximum average transmit power $p^\text{max}= 5 $ dBm.
 
\vspace{-4 mm}
\begin{figure}[H]
  \centering
  \includegraphics[scale=.42]{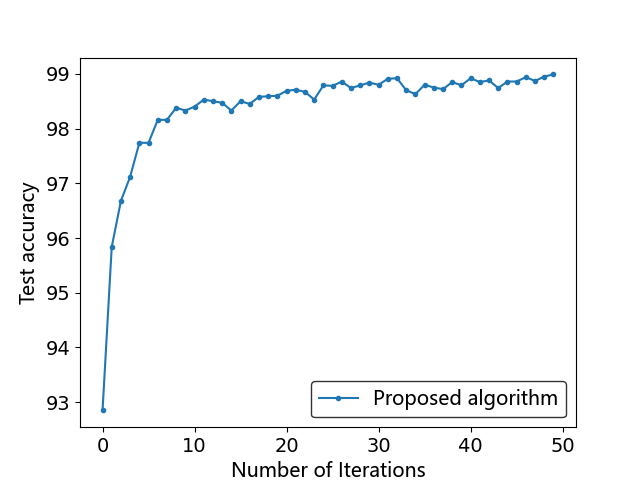} 
  \vspace{-2 mm}
  \caption{The convergence of test accuracy.} \label{accuracy}
\end{figure}

\vspace{-3 mm}
\vspace{-4 mm}
\begin{figure}[H]
  \centering
  \includegraphics[scale=.42]{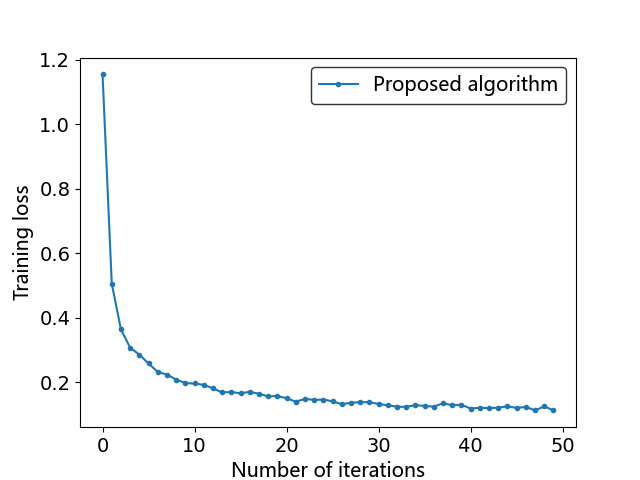} 
  \vspace{-2 mm}
  \caption{The convergence of training loss.} \label{loss}
\end{figure}
\vspace{-3 mm}

In Figs. \ref{accuracy} and \ref{loss}, we show the convergence of test accuracy and training loss of the trained model with the number of global iterations, respectively. After 50 iterations, the accuracy soars to over 99\% and the loss value plunges to less than 1\%.
\vspace{-4 mm}
\begin{figure}[H]
  \centering
  \includegraphics[scale=.42]{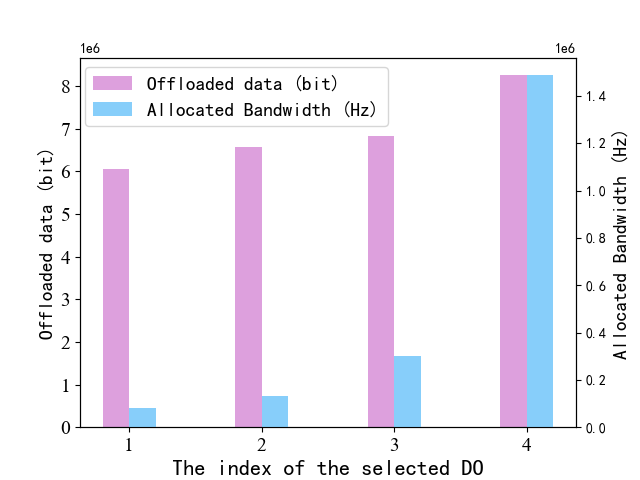} 
  \vspace{-2 mm}
  \caption{Relationship between offloaded data bits and allocated bandwidth.}\label{offloading and bandwidth}
\end{figure}
\vspace{-3 mm}
In Fig. \ref{offloading and bandwidth}, we show the data volume offloaded by the DOs and the bandwidth allocated to them. Specially, four DOs in a training session are selected for illustration. From Fig. \ref{offloading and bandwidth}, one can see that the bandwidth allocated to a DO is directly proportional to the data volume it offloading. The reason is that as the volume of offloaded data increases, the uploading requires more time, leading to a higher bandwidth allocation. This allocation aims to balance the disparities between the different transmission time of the DOs.

\vspace{-4 mm}
\begin{figure}[H]
  \centering
  \includegraphics[scale=.27]{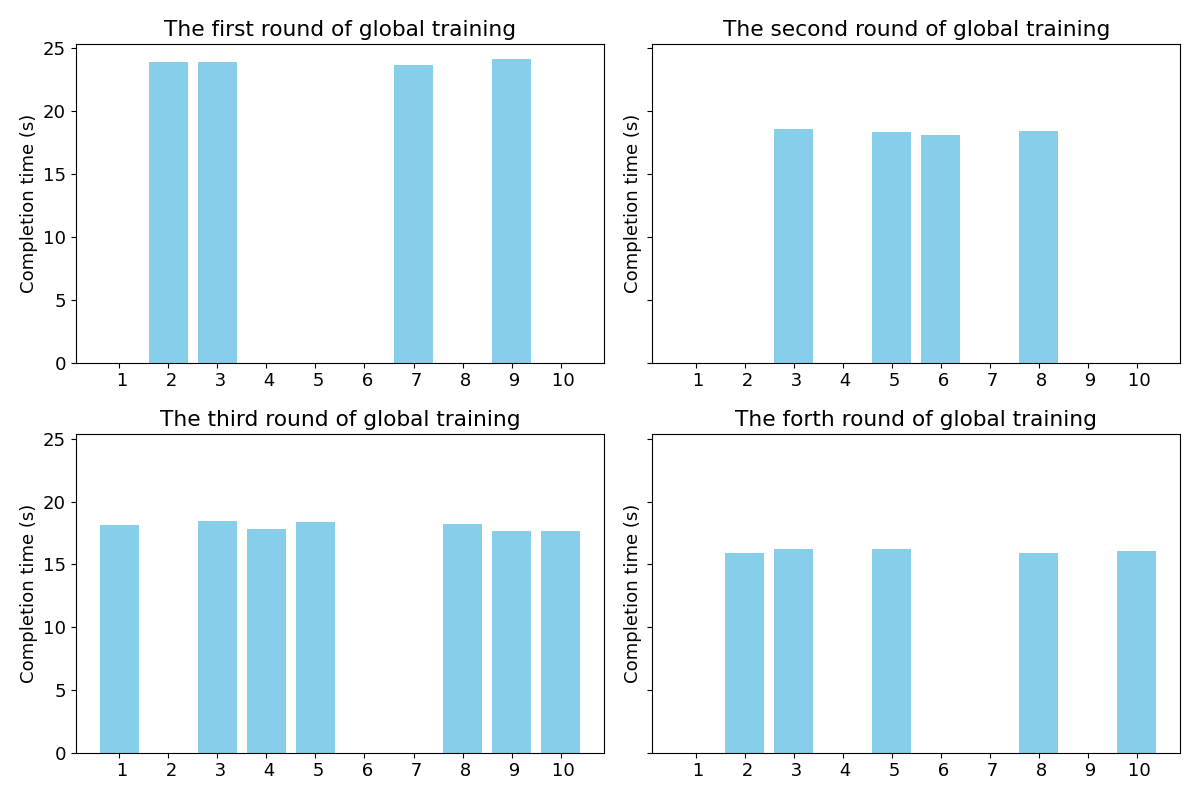} 
  \vspace{-2 mm}
  \caption{The client selection result in each round of global training.}\label{select}
\end{figure}
\vspace{-3 mm}
In Fig. \ref{select}, we show the result of client selection and the overall training time across various rounds. In the experiment, after each round of local training session, a DO may generate or finish local computing tasks, prompting a fresh round of client selection. From Fig. \ref{select}, it's noticeable that the local training time of the selected DOs keeps across the rounds. This consistency serves as the evidence of our successful client selection.

\vspace{-4 mm}
\begin{figure}[H]
  \centering
  \includegraphics[scale=.45]{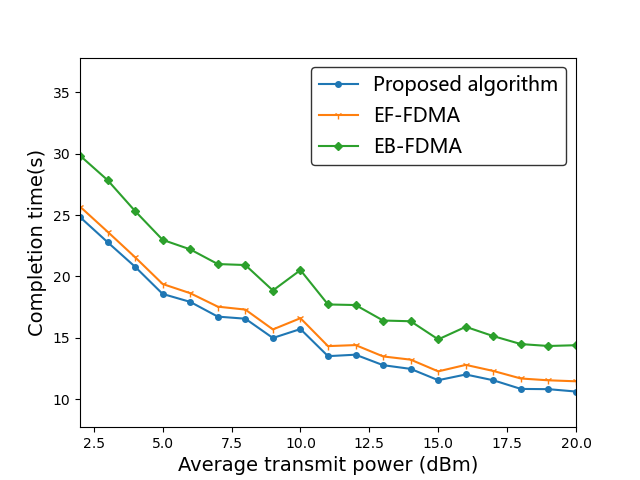} 
  \vspace{-2 mm}
  \caption{Training time versus average transmit power.}\label{equalBF}
\end{figure}
\vspace{-3 mm}

\vspace{-4 mm}
\begin{figure}[H]
  \centering
  \includegraphics[scale=.45]{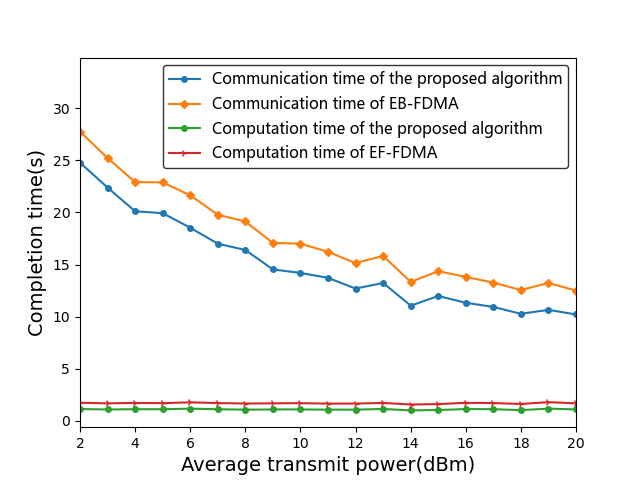} 
  \vspace{-2 mm}
  \caption{Communication and training time versus transmit power.} \label{communication and computation}
\end{figure}
\vspace{-3 mm}

Next, we conducted a comparative analysis of the proposed method against other resource allocation schemes. Specially, two benchmark schemes are implemented. The first one, termed as EF-FDMA, allocates the same computational frequency $f_1^\text{F}= ... =f_n^\text{F}$ to the DOs for performing their offloaded tasks. The second one, termed as EB-FDMA, assigns equal bandwidth $b_1 =...= b_n$ to the DOs to upload their data. In Fig.\ref{equalBF}, we illustrate the relationship between the transmit power of the DOs and the task completion time. Notably, the completion time for all algorithms decreases as the transmit power of the DOs increases. This reduce the transmission time between the DOs and the MO. In comparison, our proposed method exhibits superior performance to the other schemes. 
In Fig. \ref{communication and computation}, we show the changes in communication time (labeled 'Com') and computation time (labeled 'C') respectively, as the transmit power of each DO increases. The discernible gap between the EF-FDMA and the proposed method becomes evident.

The observed performance superiority of the proposed method is attributed to its ability to jointly optimize the communication and computing resources. The marginal difference between the EF-FDMA method and the proposed algorithm is due to the substantial disparity between transmission time and computation time. In cases where this difference is substantial, altering the frequency $f_n^\text{F}$ has minimal impact on the final completion time. This is consistent with the significant difference between transmission time and computation time, rendering adjustments to the computational frequency $f_n^\text{F}$ negligible in affecting the completion time.

\section{Conclusion}\label{sec5}

In this paper, we investigate task offloading mechanisms in federated learning. First, we introduce the motivation of using a collaborative computing framework to deal with uneven resource allocation in federated learning. Then, a collaborative computing framework is proposed to motivate rational users to participate in federated learning for the allocation of computational resources, bandwidth resources, etc. while guaranteeing accuracy in order to minimise the training time for each round. Next, we employ a low-complexity iterative algorithm to solve this optimisation problem. Finally, we discuss the simulation results, and the numerical results verify the feasibility of the scheme.

\backmatter

\begin{appendices}

\section{Proof of Lemma 1}\label{secA1}

Assuming different completion times, we can sort them as follows: $T_1 < T_2 < ... T_n = T^\ast$. Now, for $T_n$, allocating more bandwidth and more computing resources to it will reduce its time, resulting in the optimal time $T^\ast = T_{n-1}$. In the case of the min-max problem, the optimal solution depends on the worst T. Only by continuously optimizing the worst completion time until all Ts are equal, with no client able to reduce its completion time through resource allocation, can the optimal solution be achieved. Thus, the lemma is proved.

\section{Proof of Lemma 2}
\label{appendix:proof of lemma 2}

The first order and second order derivatives of $\mathcal{L}(\mathbf{F}, \lambda,\mu)$ with respect to ${f_n^\text{F}}$
respectively given by
\begin{equation}
    \frac{\partial \mathcal{L}(\mathbf{F}, \lambda,\mu)}{\partial f_n^\text{F}} = -\frac{C_nD_nI_n}{(f_n^\text{F})^2}+\lambda+\mu, \label{orderf1}
\end{equation}
and
\begin{equation}
     \frac{\partial^2\mathcal{L}(\mathbf{F}, \lambda,\mu)}{\partial (f_n^\text{F})^2} =\frac{2C_nD_nI_n}{(f_n^\text{F})^3}. \label{orderf2}
\end{equation}
From \eqref{orderf1}, we can see that in the domain of definition of $f_n^\text{F}$, $\frac{\partial^2\mathcal{L}\mathbf{F}, \lambda,\mu)}{\partial (f_n^\text{F})^2}>0$, so this dyadic function is concave and the optimal value of the original function can be obtained by Lagrange dule.
As for the optimal $(f_n^\text{F})^\ast$ of problem (18), according to KKT conditions,   
therefore, based on the first-order derivative condition , let $ \frac{\partial \mathcal{L}(\mathbf{F}, \lambda,\mu)}{\partial f_n^\text{F}} =0$,
we have 
\begin{align}
    -\frac{C_nD_nI_n}{(f_n^\text{F})^2}+\lambda+\mu=0, \\ \notag
 (f_n^\text{F})^\ast=\sqrt{\frac{C_nD_nI_n}{\lambda+\mu}}.
\end{align}

\section{Proof of Lemma 3}
\label{appendix:proof of lemma 3}
Based on the functional properties of Eq.\eqref{Tcom}, we can see that $p_n$ is its decreasing function, and next, we analyze the properties of $b_n$: 
 we can define function 
\begin{equation}
    z=\frac{1}{y\ln{\left (\frac{1}{1+y}  \right )} }, y>0.
\end{equation}
Then, we get 
\begin{align}
  & z'=\frac{\frac{1}{y+1}-\ln{(1+\frac{1}{y} )}  }{\left [ y\ln\left(\frac{1}{y+1}\right) \right ]^2 } , \label{dxy}\\ \notag
& z''=\frac{y\ln^2({1+\frac{1}{y} })+2\left [ (y+1) \left (\ln( {1+\frac{1}{y} })-\frac{1}{y+1}  \right ) \right ] ^2}{\left [y(1+y)\ln\left(\frac{1}{y+1}\right) \right ]^2 } > 0 .
\end{align}
According to \eqref{dxy}, $z''>0$ and $z'<0$  is constant, so z is an decreasing function, i.e., the object function in \eqref{Tcom} is an decreasing function of bandwidth $b_n$.

There is a functional relationship between $p_n$ and $b_n$, so we represent $b_n$ as a function of $p_n$:
\begin{equation}
    p_n(b)=\frac{N_0b_n}{g_n}\left ( 2^\frac{H_n}{T_n^\text{Tx}b_n} -1 \right ) . \label{function of pb}
\end{equation}
The first order and second order derivatives of $p_n$ can be given by
\begin{equation}
    \frac{\partial p_n}{\partial b_n} =\left ( \mathrm {e}^{\frac{(\ln_{}{2})H_n}{t_nb_n(\rho )}  } -1-\frac{(\ln_{}{2})H_n }{t_nb_n}\mathrm {e}^{\frac{(\ln_{}{2})H_n}{t_nb_n}}   \right ) ,\label{first order of bn}
\end{equation}
and
\begin{equation}
     \frac{\partial^2 p_n}{\partial b_n^2} =\frac{N_0(\ln{2})^2H_n^2 }{g_nt_n^2b_n^3}e^\frac{\ln(2)H_n}{t_nb_n}\ge 0. \label{second order of bn}
\end{equation}
According to $\frac{\partial^2 p_n}{\partial b_n^2}\ge 0$, we get that $ \frac{\partial p_n}{\partial b_n}$ is an increasing function of $b_n$. $\textstyle \lim_{b_n \to 0^+} \frac{\partial p_n}{\partial b_n}=0$, so we can get $\frac{\partial p_n}{\partial b_n}<0$. Therefore, $p_n$ in \eqref{function of pb} is a decreasing function of $b_n$. The \textbf{Lemma 3} has been proved.

\section{Proof of Lemma 4}
\label{appendix:proof of lemma 4}

Then we minimize $p_n$ to obtain optimal $b^\ast$, the optimization problem can be expressed as
\begin{align}
    \min _{b } \; & \; \frac{N_0b_n}{g_n}\left ( 2^\frac{H_n}{t_nb_n} -1 \right ), \label{minp to maximize b} \\
    \mbox{s.t.}\;\;\;\;\; &  \eqref{bn4}, \eqref{pmax4}, \\ \notag
    & b_n\ge b_n^\text{min}, \forall n \in \mathcal{N}. \label{constraint bnmin}
\end{align}
From $\frac{\partial^2 p_n}{\partial b_n^2}\ge 0$, the above minimize problem \eqref{minp to maximize b} is convex. Since $p_n(b)$ is a decreasing function of $b_n$,  When $p = p_n^\text{max}$, we get a lower bound for $b_n^\text{min}$ is 
\begin{equation}
    b_n \ge b_n^\text{min} \triangleq b_n^\text{min}=\frac{(\ln_{}{2})H_n }{T_n^\text{Tx}W\left (\frac{(\ln_{}{2}) N_0H_n}{g_np_n^\text{max} T_n^\text{Tx}} \textit {e}^{-\frac{(\ln_{}{2})N_0H_n }{g_np_n^\text{max}T_n^\text{Tx}} }\right ) - \frac{(\ln_{}{2}) N_0H_n}{g_np_n^\text{max} }}\label{solution bnmin}.
\end{equation}

By using the KKT conditions and Lagrange  method \cite{9264742}, we can solve problem \eqref{minp to maximize b}. The Lagrange function of \eqref{minp to maximize b} is
\begin{equation}
    \mathcal{F} \left ( \textbf{B},\rho \right ) = \frac{N_0b_n}{g_n}\left ( 2^\frac{H_n}{T_n^\text{Tx}b_n} -1 \right ) +\rho (\sum_{n=1}^{N}b_n-B) \label{Lagrange of b},
\end{equation}
where $\rho$ is the Lagrange multiplier. We find the first order derivative of  $ \mathcal{F} \left ( \textbf{b},\rho \right )$ and set the result to 0. Then we get $b_n(\rho )$ by \eqref{solution rho} 
\begin{equation}
    \frac{\partial \mathcal{F} \left ( \textbf{B},\rho \right )}{\partial b_n}= \frac{N_0}{g_n} \left ( \mathrm {e}^{\frac{(\ln_{}{2})H_n}{T_n^\text{Tx}b_n(\rho )}  } -1-\frac{(\ln_{}{2})H_n }{T_n^\text{Tx}b_n(\rho )}\mathrm {e}^{\frac{(\ln_{}{2})H_n}{T_n^\text{Tx}b_n(\rho )}}   \right )+\rho.
\end{equation}
and  $\rho$ satisfies
   \begin{equation}
       \sum_{n=1}^{N}\max \left \{ b_n(\rho ),b_n^\text{min} \right \}= B. \label{lagrangerho}
   \end{equation} 
The optimal solution is obtained by solving \eqref{the optimal of b}. The \textbf{Lemma 4} has been proved.




\end{appendices}


\bibliography{sn-bibliography}

\end{document}